\documentclass{article}

\usepackage{natbib}

\usepackage[preprint]{my_nips_2018}

\usepackage{mysymbols}
\usepackage{mathrsfs}
\usepackage{amssymb}
\usepackage{latexsym}
\usepackage{amsmath,amsthm}
\usepackage{amsfonts}
\usepackage{booktabs}
\usepackage{soul}
\usepackage[hyphens]{url}
\usepackage{accents}

\usepackage{stmaryrd}

\usepackage{bbm}

\newcommand{\ind}{\mathbb{I}}

\renewcommand{\P}{\mathrm{(P)}}

\newcommand{\X}{\mathcal{X}}
\newcommand{\Y}{\mathcal{Y}}
\renewcommand{\Z}{\mathcal{Z}}

\newcommand{\E}{\text{{\bf E}}}

\newcommand{\comp}{\kappa}

\newcommand{\recon}{\rho}

\renewcommand{\H}{\mathcal{H}}

\newcommand{\nats}{\mathbb{N}}

\newcommand{\Rate}{{\cal E}}
\newcommand{\UC}{\Rate_{{\rm uc}}}
\newcommand{\AgRate}{\Rate_{{\rm ag}}}

\newtheorem{theorem}{Theorem}

\newtheorem{corollary}[theorem]{Corollary}

\renewcommand{\P}{\boldsymbol{\mathrm{P}}}
\renewcommand{\E}{\boldsymbol{\mathrm{E}}}

\title{A New Lower Bound for Agnostic Learning with Sample Compression Schemes}

\author{Steve Hanneke\\\texttt{steve.hanneke@gmail.com}
\And Aryeh Kontorovich\\Ben-Gurion University\\\texttt{karyeh@bgu.ac.il}}

\begin{document}

\maketitle

\begin{abstract}
We establish a tight characterization of the worst-case rates for 
the excess risk of agnostic learning with sample compression schemes 
and for uniform convergence for agnostic sample compression schemes.
In particular, we find that the optimal rates of convergence 
for size-$k$ agnostic sample compression schemes are of the form 
$\sqrt{\frac{k \log(n/k)}{n}}$, which contrasts with agnostic learning 
with classes of VC dimension $k$, where the optimal rates are of the form 
$\sqrt{\frac{k}{n}}$.
\end{abstract}

\section{Introduction}

Compression-based arguments provide some of the simplest and tightest generalization bounds in the literature.
These are
known as {\em Occam learning} in the most general setting
\citep{MR1072253},
and
the special case of {\em sample compression}
\citet{warmuth86,MR1383093,DBLP:journals/ml/GraepelHS05,DBLP:journals/ml/FloydW95}
has been receiving a fair amount of recent attention
\citep{DBLP:journals/jacm/MoranY16,david2016supervised,DBLP:conf/colt/Zhivotovskiy17,hks18}.

As the present paper deals with lower bounds, we stress up-front that
these are {\em statistical} lower bounds (rather than, say, computational \citep{DBLP:conf/nips/GottliebKN14}
or 
communication-based
\citep{DBLP:journals/corr/abs-1711-05893}).
In the realizable case, \citet{warmuth86,DBLP:journals/ml/FloydW95} 
showed that a $k$-compression scheme
on a sample of size $n \geq e k$
achieves an expected generalization error bound of order
\beqn
\label{eq:realizable-compress}
\frac{ k \log( n / k ) }{n}
.
\eeqn
As the compression size $k$ is a rough analogue of the VC-dimension, one is immediately led to
inquire into the necessity of the $\log (n/k)$ factor. While known to be removable
from the realizable VC bound
\citep{DBLP:journals/iandc/HausslerLW94,DBLP:journals/jmlr/Hanneke16},
the $\log (n/k)$ factor in (\ref{eq:realizable-compress}) turns out to be tight
\citep{DBLP:journals/ml/FloydW95}.
On the other hand, turning to the agnostic case, 
the corresponding compression result from 
\citet{DBLP:journals/ml/GraepelHS05}
implies an upper bound on the expected excess generalization error
of a certain $k$-compression scheme
on a sample of size $n \geq e k$ by a bound of order 
\beqn
\label{eq:agnostic-compress}
\sqrt{\frac{k \log(n/k)}{n}}.
\eeqn
Here again, the agnostic VC analogue of (\ref{eq:agnostic-compress})
\citep[Theorem 4.10]{MR1741038}
might suggest that the $\log (n/k)$ factor might be superfluous.
Though it is a simpler matter to give an $\Omega(\sqrt{k/n})$
lower bound,
it proves significantly more challenging to determine whether the factor of $\log(n/k)$ 
is required for this general bound.  As our main result in this work (Section~\ref{sec:order-indep}), we prove that 
this $\log(n/k)$ factor in (\ref{eq:agnostic-compress}) \emph{cannot} be removed.
We also prove an analogous 
lower bound for order-dependent compression schemes (Section~\ref{sec:order-dep}),
where the factor becomes $\log(n)$, 
which again is tight.

\hide{
cite floyd+warmth,
our recent compression paper,
thank Nikita
}

\section{Order-Independent Compression Schemes}
\label{sec:order-indep}

Let $\Z = \X \times \Y$, where $\X$ is any nonempty set and $\Y = \{0,1\}$, 
and suppose $\X$ is equipped with a $\sigma$-algebra defining the measurable sets.
An agnostic sample compression scheme
is specified by a \emph{size} $k \in \nats$ 
and a \emph{reconstruction function} $\recon$,
which maps any (multi)set 
$\{z_{1},\ldots,z_{k^{\prime}}\} \subseteq \Z$ with $0 \leq k^{\prime} \leq k$
to a measurable function $h : \X \to \Y$.
For any $n \in \nats$ and any sequence $z_{1},\ldots,z_{n}$, 
define 
\begin{equation*}
\H_{k,\recon}(z_{1},\ldots,z_{n}) = \{ \recon( \{z_{i_{1}},\ldots,z_{i_{k^{\prime}}}\} ) : k^{\prime} \leq k, 1 \leq i_{1} < \cdots < i_{k^{\prime}} \leq n \}.
\end{equation*}

Now for any probability measure $P$ on $\Z$ and any $n \in \nats$, 
let $Z_{[n]} = \{(X_{1},Y_{1}),\ldots,(X_{n},Y_{n})\}$ be independent $P$-distributed random variables, 
and
for any classifier $h : \X \to \Y$, define $R(h;P) = P( \{(x,y) : h(x) \neq y\} )$ the \emph{error rate} of $h$, 
and define $\hat{R}(h;Z_{[n]}) = \frac{1}{n} \sum_{i=1}^{n} \ind[ h(X_{i}) \neq Y_{i} ]$ the \emph{empirical error rate} of $h$.

Now there are essentially two types of results for agnostic compression schemes in the literature: 
namely, \emph{uniform convergence} rates and \emph{agnostic learning} excess risk guarantees.
We begin with the first of these.
For any fixed agnostic sample compression scheme $(k,\recon)$, 
denote 
\begin{equation*}
\UC(n,k,\recon,P) = \E \sup_{h \in \H_{k,\recon}(Z_{[n]})} | \hat{R}(h;Z_{[n]}) - R(h;P) |.
\end{equation*}
Then, for any $n,k \in \nats$, define
\begin{equation*}
\UC(n,k) = \sup_{P, \recon} \UC(n,k,\recon,P),
\end{equation*}
where $P$ ranges over all probability measures on $\Z$, 
and $\recon$ ranges over all reconstruction functions (for the given size $k$).
For results on uniform convergence for agnostic compression schemes, 
this is the object of primary interest to this work.

It is known
(essentially from the arguments of
\citet[Theorem 2]{DBLP:journals/ml/GraepelHS05})
that for any $n,k \in \nats$ with $n \geq e k$,
\begin{equation*}
\UC(n,k) \lesssim \sqrt{\frac{k \log(n/k)}{n}}.
\end{equation*}

This upper bound is similar in form to the original bound of \citet{vapnik:71}
for uniform convergence rates for VC classes of VC dimension $k$.  However, that bound was later 
refined\footnote{
  A detailed account of the intermediate steps leading to this seminal result
  is presented in \citet{MR1741038};
  significant
  milestones include
\citet{pollard1982central,koltchinskii1981central,talagrand1994,MR1313896}.
}
to the form $\sqrt{\frac{k}{n}}$, removing the factor $\log(n/k)$.
It is therefore natural to wonder whether this same refinement might be achieved by size-$k$ agnostic sample compression schemes.
To our knowledge, this question has not previously been addressed in the literature.

The other type of results of interest for agnostic compression schemes are agnostic learning excess risk guaratnees.
Specifically, a \emph{compression function} $\comp$ is a mapping from 
any sequence $z_{1},\ldots,z_{n}$ in $\Z$ 
to an unordered sub(multi)set\footnote{An element in $S$ may repeat up to as many times as it occurs in the sequence $z_{1},\ldots,z_{n}$, 
so that $S$ effectively corresponds to picking a set of up to $k$ distinct \emph{indices} in $\{1,\ldots,n\}$ to include the corresponding $z_{i}$ points.} 
$S \subseteq \{z_{1},\ldots,z_{n}\}$ of size at most $k$.
Then, denoting $\hat{h}_{n} = \recon(\comp(Z_{[n]}))$, define 
\begin{equation*}
\AgRate(n,k,\recon,\comp,P) = \E\!\left[ R(\hat{h}_{n};P) - \min_{h \in \H_{k,\recon}(Z_{[n]})} R(h;P) \right]
\end{equation*}
and then define 
\begin{equation*}
\AgRate(n,k) = \sup_{\recon} \inf_{\comp} \sup_{P} \AgRate(n,k,\recon,\comp,P),
\end{equation*}
where again $P$ ranges over all probability measures on $\Z$
and $\recon$ ranges over all reconstruction functions (for the given size $k$), 
and where $\comp$ ranges over all compression functions (for the given size $k$).

By a standard argument, if we specify $\comp$ so as to always minimize the empirical error rate $\hat{R}(\recon(\comp(Z_{[n]})))$, 
then the excess error rate can be bounded by twice the uniform convergence bound, 
which immediately implies 
\begin{equation}
\label{eqn:ag-vs-uc}
\AgRate(n,k) \leq 2\UC(n,k).
\end{equation}
An immediate implication from above is then that any $n,k$ with $n \geq e k$ has 
\begin{equation*}
\AgRate(n,k) \lesssim \sqrt{\frac{k \log(n/k)}{n}}.
\end{equation*}

Here again, this bound is of the same form originally proven by \citet{vapnik:71} for empirical risk minimization in 
classes of VC dimension $k$, which was later refined to a sharp bound of order $\sqrt{k/n}$ \citep[Theorem 4.10]{MR1741038}.  
As such, it is again natural to ask whether the $\log(n/k)$ factor in the above bound for agnostic sample compression 
can be reduced to a constant, or is in fact necessary.
Our main contribution in this work is a construction showing that this log factor is indeed necessary, as stated in the following results.
In all of the results below, $c$ represents a numerical constant, whose value must be set sufficiently large (as discussed in the proofs) for the results to hold.

\begin{theorem}
\label{thm:ag-lb}
For any $n,k \in \nats$ with $|\X| \geq n \geq c k$, 
\begin{equation*}
\AgRate(n,k) \gtrsim \sqrt{\frac{k \log(n/k)}{n}}.
\end{equation*}
\end{theorem}

By the relation \eqref{eqn:ag-vs-uc} discussed above, 
between uniform convergence and agnostic learning by empirical risk minimization over $\H_{k,\recon}(Z_{[n]})$, 
this also has the following immediate implication.

\begin{theorem}
\label{thm:uc-lb}
For any $n,k \in \nats$ with $|\X| \geq n \geq c k$,
\begin{equation*}
\UC(n,k) \gtrsim \sqrt{\frac{k \log(n/k)}{n}}.
\end{equation*}
\end{theorem}

Together with the known upper bounds mentioned above, this provides a tight characterization of the worst-case 
rate of uniform convergence for agnostic sample compression schemes.

\begin{corollary}
\label{cor:uc-tight}
For any $n,k \in \nats$ with $|\X| \geq n \geq c k$,  
\begin{equation*}
\AgRate(n,k) \asymp \sqrt{\frac{k \log(n/k)}{n}}
\end{equation*}
and
\begin{equation*}
\UC(n,k) \asymp \sqrt{\frac{k \log(n/k)}{n}}.
\end{equation*}
\end{corollary}

We now present the proof of Theorem~\ref{thm:ag-lb}.

\begin{proof}[Proof of Theorem~\ref{thm:ag-lb}]
Fix any $n,k \in \nats$ with $|\X| \geq n \geq c k$ for a sufficiently large numerical constant $c \geq 4$ (discussed below), 
denote $m = 2^{\lfloor \log_{2}( n / k ) \rfloor}$, 
and let $x_{0},\ldots,x_{k m-1}$ denote any $k m$ distinct elements of $\X$.
For simplicity, suppose $m/\log_{2}(m) \in \nats$ 
(the argument easily extends to the general case by introducing floor functions, with only the numerical constants changing in the final result).
The essential strategy behind our construction is to create an embedded instance of a 
construction for proving the lower bound for agnostic learning in VC classes, 
where here the VC dimension of the embedded scenario will be $k \log_{2}(m)$.
The construction of this embedded scenario is our starting point.  From there we 
also need to argue that there is a function contained in $\H_{k,\recon}(Z_{[n]})$ with risk 
not too much larger than the best classifier in the embedded VC class, which allows us to extend the 
lower bound argument for the embedded VC class to compression schemes.
For any $0 \leq i \leq m-1$, let $b_{j}(i)$ denote the $(j+1)^{{\rm th}}$ bit of $i$ in the binary representation of $i$: 
that is, $i = \sum_{j=0}^{\log_{2}(m)-1} b_{j}(i) 2^{j}$, with $b_{0}(i),\ldots,b_{\log_{2}(m)-1}(i) \in \{0,1\}$.

We construct the reconstruction function based on $k$ ``blocks'', each with $m / \log_{2}(m)$ ``sub-blocks''.
Specifically, for each $t \in \{1,\ldots,k\}$, define a block $B_{t} = \{ (t-1)m, \ldots, tm - 1 \}$, 
and for each $s \in \{1,\ldots, m/\log_{2}(m) \}$, 
define a sub-block 
\begin{equation*}
B_{ts} = \{ (t-1)m + (s-1)\log_{2}(m), \ldots, (t-1)m + s\log_{2}(m) - 1 \}.
\end{equation*}
Then for any
$i \in B_{t}$ and $t \in \{1,\ldots,k\}$, 
define $h_{t,i} : \X \to \Y$ as any function satisfying the property that, 
for $j = (t-1)m + (s-1)\log_{2}(m) + r \in B_{ts}$ 
(for any $s \in \{1,\ldots, m/\log_{2}(m)\}$ and $r \in \{0,\ldots,\log_{2}(m)-1\}$),
\begin{equation*}
h_{t,i}(x_{j}) = b_{r}(i-(t-1)m).
\end{equation*}
Thus, the subsequence of $x_{j}$ points corresponding to the indices $j$ within each sub-block $B_{ts}$ 
have $h_{t,i}(x_{j})$ values corresponding to the bits of the integer $i-(t-1)m$, 
and this repeats identically for every sub-block $B_{ts}$ in the block $B_{t}$.

Now we construct a reconstruction function $\recon$ that outputs functions which correspond to some such $h_{t,i}$ function 
within each block $B_{t}$, but potentially using a different bit pattern $i-(t-1)m$ for each $t$.
Formally, for any
$i_{1},\ldots,i_{k} \in \nats \cup \{0\}$ with $i_{t} \in B_{t}$ (for each $t \in \{1,\ldots,k\}$),
and any $y_{1},\ldots,y_{k} \in \Y$,
define $\recon(\{(x_{i_{1}},y_{1}),\ldots,(x_{i_{k}},y_{k})\}) = \tilde{h}_{i_{1},\ldots,i_{k}}$, 
where $\tilde{h}_{i_{1},\ldots,i_{k}} : \X \to \Y$ is any function satisfying the property that
each $t \in \{1,\ldots,k\}$ and $j \in \{(t-1)m,\ldots,tm-1\}$ has 
$\tilde{h}_{i_{1},\ldots,i_{k}}(x_{j}) = h_{t,i_{t}}(x_{j})$: that is, the points $x_{i_{t}}$ in the compression set 
are interpreted by the compression scheme as encoding the desired \emph{label sequence} for sub-blocks $B_{ts}$ in the \emph{bits} of $i_{t}-(t-1)m$.
For our purposes, $\tilde{h}_{i_{1},\ldots,i_{k}}(x)$ may be defined arbitrarily for $x \in \X \setminus \{x_{0},\ldots,x_{km-1}\}$.
Note that $\recon(\{(x_{i_{1}},y_{1}),\ldots,(x_{i_{k}},y_{k})\})$ is invariant to the $y_{1},\ldots,y_{k}$ values, 
so for brevity we will drop the $y_{i}$ arguments and simply write $\recon(\{x_{i_{1}},\ldots,x_{i_{k}}\})$
(this is often referred to as an \emph{unlabeled} compression scheme in the literature).
For completeness, $\recon(S)$ should also be defined for sets $S \subseteq \X$ of size at most $k$ 
that do not have exactly one element $x_{i}$ with $i \in B_{t}$ for every $t$; 
for our purposes, let us suppose that in these cases, for every $t$ with 
$S \cap \{x_{i} : i \in B_{t}\} \neq \emptyset$, let $i_{t} = \min\{ i \in B_{t} : x_{i} \in S \}$, 
and for every $t$ with $S \cap \{x_{i} : i \in B_{t}\} = \emptyset$, let $i_{t} = (t-1)m$;
then define $\recon(S) = \tilde{h}_{i_{1},\ldots,i_{k}}$.  In this way, $\recon(S)$ is defined for 
all $S \subseteq \X$ with $|S| \leq k$.

Now define a family of distributions $P^{(\sigma)}$, $\sigma = \{\sigma_{t,r}\}$, 
with $\sigma_{t,r} \in \{-1,1\}$ for $t \in \{1,\ldots,k\}$ and $r \in \{0,\ldots,\log_{2}(m)-1\}$, as follows.
Every $P^{(\sigma)}$ has marginal $P_{X}$ on $\X$ uniform on $x_{0},\ldots,x_{km-1}$, 
and for each $j = (t-1)m + (s-1)\log_{2}(m) + r \in B_{ts}$ 
(for $t \in \{1,\ldots,k\}$, $s \in \{1,\ldots, m/\log_{2}(m)\}$, and $r \in \{0,\ldots,\log_{2}(m)-1\}$)
set $P^{(\sigma)}(Y=1|X=x_{j}) = \frac{1}{2} + \frac{\epsilon}{2} \sigma_{t,r}$,
where 
\begin{equation*}
\epsilon = \sqrt{\frac{k \log_{2}(m)}{n}}.
\end{equation*}
Now let us suppose $\sigma$ is chosen \emph{randomly}, with $\sigma_{t,r}$ independent ${\rm Uniform}(\{-1,1\})$.
Then (since max $\geq$ average) note that choosing $P = P^{(\sigma)}$ 
now results in 
\begin{equation*}
\AgRate(n,k) \geq \E\!\left[ \inf_{\comp} \AgRate(n,k,\recon,\comp,P^{(\sigma)}) \right],
\end{equation*}
so that it suffices to study the expectation on the right hand side.

As mentioned, the purpose of this construction is to create an embedded instance of a scenario 
that witnesses the lower bound for agnostic learning in VC classes, where the VC dimension 
of the embedded scenario here is $k \log_{2}(m)$.  Specifically, in our construction, for any $t \in \{1,\ldots,k\}$ 
and $r \in \{0,\ldots,\log_{2}(m)-1\}$, denoting by 
\begin{equation*}
C_{t,r} = \{ (t-1)m + (s-1)\log_{2}(m) + r : s \in \{1,\ldots,m/\log_{2}(m)\}\},
\end{equation*}
the locations $\{ x_{j} : j \in C_{t,r} \}$ 
together essentially represent a single location in the embedded problem: that is, 
their $h_{t,i}(x_{j})$ values are bound together, as are their $P(Y=1|X=x_{j})$ values.
However, this itself is not sufficient to supply a lower bound, since the constructed scenario 
exists only in the \emph{complete} space of possible reconstructions 
$\H_{k,\recon}^{*} = \{ \recon(\{x_{i_{1}},\ldots,x_{i_{k}}\}) : i_{1},\ldots,i_{k} \in \{0,\ldots,km-1\}\}$, 
and it is entirely possible that
$\min_{h \in \H_{k,\recon}(Z_{[n]})} R(h;P) > \min_{h \in \H_{k,\recon}^{*}} R(h;P)$: 
that is, the smallest error rate achievable in $\H_{k,\recon}(Z_{[n]})$ can conceivably be 
significantly larger than the smallest error rate achievable in the embedded VC class, 
so that compression schemes in this scenario do not automatically inherit the lower bounds for the constructed VC class.
To account for this, we will study a decomposition of the construction into $k$ subproblems, 
corresponding to the $k$ blocks $B_{t}$ in the construction, and we will 
argue that within these subproblems there remains in $\H_{k,\recon}(Z_{[n]})$ a function 
with optimal predictions on \emph{most} of the points, and then stitch these functions 
together to argue that there do exist functions in $\H_{k,\recon}(Z_{[n]})$ having near-optimal 
error rates relative to the best in $\H_{k,\recon}^{*}$.

Specifically, fix any $t \in \{1,\ldots,k\}$ and let $P_{t}^{(\sigma)}$ denote the conditional distribution of $(X,Y) \sim P^{(\sigma)}$ 
given $\sigma$ and the event that $X \in \{ x_{j} : j \in B_{t} \}$.  Also denote 
$\H_{t}^{*} = \{ h_{t,i} : i \in B_{t} \}$, 
$i_{t}^{*} = \argmin_{i \in B_{t}} R(h_{t,i};P_{t}^{(\sigma)})$,
$h_{t}^{*} = h_{t,i_{t}^{*}}$, 
and 
\begin{equation*}
\H_{t}(Z_{[n]}) = \{ h_{t,i} : i \in B_{t}, x_{i} \in \{x_{(t-1)m},X_{1},\ldots,X_{n}\} \}.
\end{equation*}
These correspond to the classifications of block $t$ realizable by classifiers in $\H_{k,\recon}(Z_{[n]})$ 
(where the addition of the $x_{(t-1)m}$ point to the data set is due to our specification of $\recon(S)$ for 
sets $S$ that contain no elements $x_{i}$ with $i \in B_{t}$, so that classifying block $t$ according to $h_{t,(t-1)m}$ 
is always possible).
There are now two components at this stage in the argument: 
first, that any compression function $\comp$ 
results in $\hat{h} = \recon(\comp(Z_{[n]}))$ with 
$\E[ R(\hat{h};P_{t}^{(\sigma)}) - R(h_{t}^{*};P_{t}^{(\sigma)}) ] \geq \epsilon/(8 e^{4})$, 
and second, that 
$\E[ \min_{h \in \H_{t}(Z_{[n]})} R(h;P_{t}^{(\sigma)}) - R(h_{t}^{*};P_{t}^{(\sigma)}) ] \leq \epsilon / (16 e^{4})$.

For the first part, note that for any $r \in \{0,\ldots,\log_{2}(m)-1\}$, 
for any $j \in C_{t,r}$, $h_{t}^{*}(x_{j}) = \frac{\sigma_{t,r}+1}{2}$.
Furthermore, for any compression function $\comp$, 
note that any $\hat{h}$ that $\recon(\comp(Z_{[n]}))$ is capable of producing 
has $\hat{h}(x_{j}) = \hat{h}(x_{j^{\prime}})$ 
for every $j,j^{\prime} \in C_{t,r}$.
In particular, if we let $\hat{i}_{t} \in B_{t}$ be the index with 
$b_{r}(\hat{i}_{t}-(t-1)m) = \hat{h}(x_{(t-1)m+r})$ for every $r \in \{0,\ldots,\log_{2}(m)-1\}$, 
then $\hat{h}$ and $h_{t,\hat{i}_{t}}$ agree on every element of $\{x_{j} : j \in B_{t}\}$.
This also implies 
\begin{align*}
R(\hat{h};P_{t}^{(\sigma)}) - R(h_{t}^{*};P_{t}^{(\sigma)})
& = R(h_{t,\hat{i}_{t}};P_{t}^{(\sigma)}) - R(h_{t}^{*};P_{t}^{(\sigma)})
\\ & = \frac{1}{\log_{2}(m)} \sum_{r=0}^{\log_{2}(m)-1} \epsilon \ind\!\left[ b_{r}(\hat{i}_{t}-(t-1)m) \neq \frac{\sigma_{t,r}+1}{2} \right].
\end{align*}
Therefore, denoting by $n_{t,r} = |\{ i \leq n : X_{i} \in \{ x_{j} : j \in C_{t,r} \} \}|$, we have 
\begin{equation*}
\E[ R(\hat{h};P_{t}^{(\sigma)}) - R(h_{t}^{*};P_{t}^{(\sigma)}) ] 
= \frac{\epsilon}{\log_{2}(m)} \!\!\!\sum_{r=0}^{\log_{2}(m)-1}\!\!\! \E\!\left[ \P\!\left( b_{r}(\hat{i}_{t}-(t-1)m) \neq \frac{\sigma_{t,r}+1}{2} \middle| n_{t,r} \right) \right].
\end{equation*}
For any given $r \in \{0,\ldots,\log_{2}(m)-1\}$, 
enumerate the $n_{t,r}$ random variables $(X_{i},Y_{i})$ with $X_{i} \in \{ x_{j} : j \in C_{t,r} \}$ as $(X_{i(r,1)},Y_{i(r,1)}),\ldots,(X_{i(r,n_{t,r})},Y_{i(r,n_{t,r})})$,
and note that given $n_{t,r}$, the values $(Y_{i(r,1)},\ldots,Y_{i(r,n_{t,r})})$ are a \emph{sufficient statistic} for $\sigma_{t,r}$ (see Definition 2.4 of \cite{schervish:95}), 
and therefore (see Theorem 3.18 of \cite{schervish:95}) there exists a (randomized) decision rule $\hat{f}_{t,r}(Y_{i(r,1)},\ldots,Y_{i(r,n_{t,r})})$ depending only on these variables 
and independent random bits such that
\begin{equation*}
\P\!\left( b_{r}(\hat{i}_{t}-(t-1)m) \neq \frac{\sigma_{t,r}+1}{2} \middle| n_{t,r} \right)
= \P\!\left( \hat{f}_{t,r}(Y_{i(r,1)},\ldots,Y_{i(r,n_{t,r})}) \neq \frac{\sigma_{t,r}+1}{2} \middle| n_{t,r} \right).
\end{equation*}
Furthermore, by Lemma 5.1 of \cite{MR1741038}\footnote{
  The lower bound in \cite[Lemma 5.1]{MR1741038} relied on Slud's lemma;
  the analysis has since been tightened to yield asymptotically optimal lower bounds
  \citep{DBLP:journals/corr/KontorovichP16}.
  }, we have 
\begin{equation*}
\P\!\left( \hat{f}_{t,r}(Y_{i(r,1)},\ldots,Y_{i(r,n_{t,r})}) \neq \frac{\sigma_{t,r}+1}{2} \middle| n_{t,r} \right) 
> \frac{1}{8e} \exp\!\left\{ - (8/3) n_{t,r} \epsilon^{2} \right\}.
\end{equation*}
Altogether, and combined with Jensen's inequality, we have that 
\begin{align*}
& \E[ R(\hat{h};P_{t}^{(\sigma)}) - R(h_{t}^{*};P_{t}^{(\sigma)}) ] 
\\ & \geq \frac{\epsilon}{8e\log_{2}(m)} \sum_{r=0}^{\log_{2}(m)-1} \E\!\left[ \exp\!\left\{ - (8/3) n_{t,r} \epsilon^{2} \right\} \right]
\geq \frac{\epsilon}{8e\log_{2}(m)} \sum_{r=0}^{\log_{2}(m)-1} \exp\!\left\{ - (8/3) \E[ n_{t,r} ] \epsilon^{2} \right\}
\\ & = \frac{\epsilon}{8e\log_{2}(m)} \sum_{r=0}^{\log_{2}(m)-1} \exp\!\left\{ - (8/3) \frac{n}{k\log_{2}(m)} \epsilon^{2} \right\}
\geq \frac{\epsilon}{8e\log_{2}(m)} \sum_{r=0}^{\log_{2}(m)-1} e^{ - (8/3) }
\geq \frac{\epsilon}{8 e^{4}}.
\end{align*}

Now for the second part, for any $x \in \{x_{i} : i \in \{0,\ldots,km-1\}\}$, denote by $I(x)$ the index $i$ such that $x = x_{i}$.
Note that an $i$ for which $h_{t,i}$ has minimal $R(h_{t,i};P_{t}^{(\sigma)})$ among all $h_{t,i^{\prime}} \in \H_{t}(Z_{[n]})$
can equivalently be defined as an $i$ 
with minimal $\sum_{j=0}^{\log_{2}(m)-1} \ind[ b_{j}(i-(t-1)m) \neq b_{j}(i_{t}^{*}-(t-1)m) ]$ 
among all $i^{\prime} \in B_{t} \cap \{ I(X_{1}),\ldots,I(X_{n}),(t-1)m \}$, 
and furthermore, for such an $i$, 
\begin{equation*}
R(h_{t,i};P_{t}^{(\sigma)}) - R(h_{t}^{*};P_{t}^{(\sigma)}) 
= \frac{\epsilon}{\log_{2}(m)} \sum_{j=0}^{\log_{2}(m)-1} \ind[ b_{j}(i-(t-1)m) \neq b_{j}(i_{t}^{*}-(t-1)m) ].
\end{equation*}
For any $i \in B_{t}$, denote 
\begin{equation*}
\Delta_{t}(i) = \sum_{j=0}^{\log_{2}(m)-1} \ind[ b_{j}(i-(t-1)m) \neq b_{j}(i_{t}^{*}-(t-1)m) ].
\end{equation*}
Thus, it suffices to establish the stated upper bound for the quantity
\begin{equation*}
\frac{\epsilon}{\log_{2}(m)} \E\!\left[ \min_{i \in B_{t} \cap \{ I(X_{1}),\ldots,I(X_{n}),(t-1)m \}} \Delta_{t}(i) \right].
\end{equation*}
Now consider a random variable $X \sim P_{X}(\cdot | \{x_{i} : i \in B_{t})$: 
that is, $X$ has distribution the same as the marginal of $P_{t}^{(\sigma)}$ on $\X$.
Then note that the conditional distribution of $\Delta_{t}(I(X))$ given $\sigma$ 
is ${\rm Binomial}(\log_{2}(m),\frac{1}{2})$.
Let $q = 16 e^{4}$, and suppose the numerical constant $c$ is sufficiently large so that $q \leq (1/2)\log_{2}(m)$.
Then we have 
\begin{align*}
& \P\!\left( \Delta_{t}(I(X)) \leq \frac{1}{2q}\log_{2}(m) \middle| \sigma \right) 
= \sum_{\ell=0}^{ \lfloor (1/2q) \log_{2}(m) \rfloor } \binom{\log_{2}(m)}{\ell} \frac{1}{m} 
\\ & \geq \frac{1}{m} \left( \frac{\log_{2}(m)}{\lfloor (1/2q) \log_{2}(m) \rfloor} \right)^{\lfloor (1/2q) \log_{2}(m) \rfloor}
\geq \frac{1}{m} ( 4q )^{(1/2q) \log_{2}(m)}
= m^{(1/2q)\log_{2}(4q) - 1}.
\end{align*}
Thus, by independence of the samples $X_{1},\ldots,X_{n}$, 
denoting $n_{t} = |\{ i \leq n : X_{i} \in \{ x_{j} : j \in B_{t} \} \}|$, 
we have 
\begin{align*}
& \P\!\left( \min_{i \in B_{t} \cap \{I(X_{1}),\ldots,I(X_{n}),(t-1)m\}} \Delta_{t}(i) > \frac{1}{2q}\log_{2}(m) \middle| \sigma, n_{t} \right) 
\\ & \leq \P\!\left( \forall i \in B_{t} \cap \{I(X_{1}),\ldots,I(X_{n})\}, \Delta_{t}(i) > \frac{1}{2q}\log_{2}(m) \middle| \sigma, n_{t} \right)
\\ & = \P\!\left( \Delta_{t}(I(X)) > \frac{1}{2q}\log_{2}(m) \middle| \sigma \right)^{n_{t}}
\\ & \leq \left( 1 - m^{(1/2q)\log_{2}(4q)-1} \right)^{n_{t}}
\leq  \exp\!\left\{  - m^{(1/2q)\log_{2}(4q)-1} n_{t} \right\}.
\end{align*}
Altogether, by the law of total expectation, 
and using the fact that $R(h;P_{t}^{(\sigma)}) \leq 1$, 
we have established that 
\begin{equation*}
\E\!\left[ \min_{h \in \H_{t}(Z_{[n]})} R(h;P_{t}^{(\sigma)}) - R(h_{t}^{*};P_{t}^{(\sigma)}) \right] 
\leq \frac{\epsilon}{2q} + \E\!\left[ \exp\!\left\{  - m^{(1/2q)\log_{2}(4q)-1} n_{t} \right\} \right].
\end{equation*}
Since $n_{t}$ is a ${\rm Binomial}(n, 1/k)$ random variable, the rightmost term evaluates to the moment generating function of this distribution: 
that is, 
\begin{align*}
& \E\!\left[ \exp\!\left\{  - m^{(1/2q)\log_{2}(4q)-1} n_{t} \right\} \right]
= \left( 1 - \frac{1}{k} + \frac{1}{k} \exp\!\left\{ - m^{(1/2q)\log_{2}(4q)-1} \right\} \right)^{n}
\\ & \leq \max\!\left\{ 2 \left( 1 - \frac{1}{k} \right)^{n}, 2 \left( \frac{1}{k} \right)^{n} \exp\!\left\{ - m^{(1/2q)\log_{2}(4q)-1} n \right\} \right\}
\\ & \leq \max\!\left\{ 2 e^{-n/k}, 2 \exp\!\left\{ - m^{(1/2q)\log_{2}(4q)} \right\} \right\}
\\ & = \max\!\left\{ 2 e^{-n/k}, 2 \left( \exp\!\left\{ - (1/2q)\log_{2}(4q) m^{(1/2q)\log_{2}(4q)} \right\} \right)^{\frac{2q}{\log_{2}(4q)}} \right\}
\\ & \leq \max\!\left\{ 2 e^{-n/k}, 2 \left( \frac{2q}{\log_{2}(4q)} \right)^{\frac{2q}{\log_{2}(4q)}} \frac{1}{m} \right\}.
\end{align*}
Since both of these terms shrink strictly faster than the above specification of $\epsilon$ as a function of $n/k$, 
and therefore, for a sufficiently large choice of the numerical constant $c$, 
both of these terms are smaller than $\frac{\epsilon}{32 e^{4}}$.
Therefore, we conclude that 
\begin{equation*}
\E\!\left[ \min_{h \in \H_{t}(Z_{[n]})} R(h;P_{t}^{(\sigma)}) - R(h_{t}^{*};P_{t}^{(\sigma)}) \right] \leq \frac{\epsilon}{16 e^{4}},
\end{equation*}
as claimed.

Together, these two components imply that
\begin{align*}
& \E\!\left[ R(\hat{h};P_{t}^{(\sigma)}) - \min_{h \in \H_{t}(Z_{[n]})} R(h;P_{t}^{(\sigma)}) \right]
\\ & = \E\!\left[ R(\hat{h};P_{t}^{(\sigma)}) - R(h_{t}^{*};P_{t}^{(\sigma)})\right] - \E\!\left[ \min_{h \in \H_{t}(Z_{[n]})} R(h;P_{t}^{(\sigma)}) - R(h_{t}^{*};P_{t}^{(\sigma)}) \right]
\geq \frac{\epsilon}{16 e^{4}}.
\end{align*}

Finally, it is time to combine these results for the individual $B_{t}$ blocks 
into a global statement about $P^{(\sigma)}$.  In particular, note that any $h$ has 
$R(h;P^{(\sigma)}) = \frac{1}{k} \sum_{t=1}^{k} R(h;P^{(\sigma)}_{t})$.
Also note that any $h$ that $\recon$ is capable of producing from arguments that are subsets of $\{X_{1},\ldots,X_{n}\}$ 
can be represented as $h = \tilde{h}_{i_{1},\ldots,i_{k}}$ for some $i_{1},\ldots,i_{k}$ where 
every $t \in \{1,\ldots,k\}$ has $i_{t} \in B_{t}$ and $x_{i_{t}} \in \{X_{1},\ldots,X_{n},x_{(t-1)m}\}$
(where the addition of the $x_{(t-1)m}$ covers the case that the set does not include any $x_{i}$ with $i \in B_{t}$, as we defined that case above).
Furthermore, every function $\tilde{h}_{i_{1},\ldots,i_{k}}$ with $i_{t}$ values satisfying these conditions \emph{can} be realized by 
$\recon$ using an argument $S$ that is a subset of $\{X_{1},\ldots,X_{n}\}$ of size at most $k$: 
namely, the set $\{ x_{i_{t}} : t \in \{1,\ldots,k\}, i_{t} \neq (t-1)m \} \subseteq \{X_{1},\ldots,X_{n}\}$.
Therefore, 
\begin{align*}
& \min_{h \in \H_{k,\recon}(Z_{[n]})} R(h;P^{(\sigma)}) 
= \min_{\substack{(i_{1},\ldots,i_{k}) \in B_{1} \times \cdots \times B_{k} :\\ \{x_{i_{1}},\ldots,x_{i_{k}}\} \subseteq \{X_{1},\ldots,X_{n}\}\cup\{x_{(t-1)m} : t \leq k\}}} R(\tilde{h}_{i_{1},\ldots,i_{k}};P^{(\sigma)})
\\ & = \min_{\substack{(i_{1},\ldots,i_{k}) \in B_{1} \times \cdots \times B_{k} :\\ \{x_{i_{1}},\ldots,x_{i_{k}}\} \subseteq \{X_{1},\ldots,X_{n}\}\cup\{x_{(t-1)m} : t \leq k\}}} \frac{1}{k} \sum_{t=1}^{k} R(h_{t,i_{t}};P^{(\sigma)}_{t})
\\ & = \frac{1}{k} \sum_{t=1}^{k} \min_{\substack{i_{t} \in B_{t} :\\ x_{i_{t}} \in \{X_{1},\ldots,X_{n},x_{(t-1)m}\}}} R(h_{t,i_{t}};P^{(\sigma)}_{t})
= \frac{1}{k} \sum_{t=1}^{k} \min_{h \in \H_{t}(Z_{[n]})} R(h;P^{(\sigma)}_{t}).
\end{align*}
Thus, for any compression function $\comp$, 
denoting $\hat{h} = \recon(\comp(Z_{[n]}))$, 
\begin{align*}
& \E\!\left[ R(\hat{h};P^{(\sigma)}) - \min_{h \in \H_{k,\recon}(Z_{[n]})} R(h;P^{(\sigma)}) \right] 
\\ & \geq \frac{1}{k} \sum_{t=1}^{k} \E\!\left[ R(\hat{h};P^{(\sigma)}_{t}) - \min_{h \in \H_{t}(Z_{[n]})} R(h;P^{(\sigma)}_{t}) \right]
\geq \frac{1}{16 e^{4}} \epsilon 
\gtrsim \sqrt{\frac{k \log(n/k)}{n}}.
\end{align*}
\end{proof}

\section{Order-Dependent Compression Schemes}
\label{sec:order-dep}

The above construction shows that the well-known $\sqrt{\frac{k \log(n/k)}{n}}$ upper bound for 
agnostic compression schemes is sometimes tight.  Note that, in the definition of agnostic compression schemes, 
we required that the reconstruction function $\recon$ take as input a (multi)set.
This type of compression scheme is often referred to as being \emph{permutation invariant}, 
since the compression set argument is unordered 
(or equivalently $\recon$ does not depend on the order of elements in its argument).

We can also show a related result for the case of \emph{order-dependent} compression schemes.
An order-dependent agnostic sample compression scheme
is specified by a \emph{size} $k \in \nats$ 
and an \emph{order-dependent reconstruction function} $\recon$,
which maps any ordered \emph{sequence} 
$(z_{1},\ldots,z_{k^{\prime}}) \in \Z^{k^{\prime}}$ with $0 \leq k^{\prime} \leq k$
to a measurable function $h : \X \to \Y$.
For any $n \in \nats$ and any sequence $z_{1},\ldots,z_{n}$, 
define 
\begin{equation*}
\H_{k,\recon}(z_{1},\ldots,z_{n}) = \{ \recon( (z_{i_{1}},\ldots,z_{i_{k^{\prime}}}) ) : k^{\prime} \leq k, i_{1},\ldots,i_{k^{\prime}} \in \{1,\ldots,n\} \}.
\end{equation*}

Now for any probability measure $P$ on $\Z$ and any $n \in \nats$, 
continuing the notation from above, 
for any fixed order-dependent agnostic sample compression scheme $(k,\recon)$, 
as above denote 
\begin{equation*}
\UC^{o}(n,k,\recon,P) = \E \sup_{h \in \H_{k,\recon}(Z_{1},\ldots,Z_{n})} | \hat{R}(h;Z_{[n]}) - R(h;P) |,
\end{equation*}
and for any $n,k \in \nats$, define
\begin{equation*}
\UC^{o}(n,k) = \sup_{P, \recon} \UC^{o}(n,k,\recon,P),
\end{equation*}
where $P$ ranges over all probability measures on $\Z$, 
and $\recon$ ranges over all order-dependent reconstruction functions (for the given size $k$).

It is known \citep{DBLP:journals/ml/GraepelHS05} 
that for any $n,k \in \nats$,
\begin{equation*}
\UC^{o}(n,k) \lesssim \sqrt{\frac{k \log(n)}{n}}.
\end{equation*}
In comparison with the above upper bound for permutation-invariant compression schemes, 
this bound has a factor $\log(n)$ in place of $\log(n/k)$.

Similarly, we can also define analogous quantities for agnostic learning excess risk guarantees.
Specifically, in this context, an \emph{ordered} compression function $\comp$ is a mapping 
from any sequence $z_{1},\ldots,z_{n}$ in $\Z$ 
to an \emph{ordered sequence} $S = (z_{i_{1}},\ldots,z_{i_{k^{\prime}}})$ for some $k^{\prime} \leq k$ and $i_{1},\ldots,i_{k^{\prime}} \in \{1,\ldots,n\}$.
Then, denoting $\hat{h}_{n} = \recon(\comp(Z_{[n]}))$, define 
\begin{equation*}
\AgRate^{o}(n,k,\recon,\comp,P) = \E\!\left[ R(\hat{h}_{n};P) - \min_{h \in \H_{k,\recon}(Z_{[n]})} R(h;P) \right]
\end{equation*}
and then define 
\begin{equation*}
\AgRate^{o}(n,k) = \sup_{\recon} \inf_{\comp} \sup_{P} \AgRate(n,k,\recon,\comp,P),
\end{equation*}
where again $P$ ranges over all probability measures on $\Z$
and $\recon$ ranges over all order-dependent reconstruction functions (for the given size $k$), 
and where $\comp$ ranges over all ordered compression functions (for the given size $k$).

By the same standard argument involving empirical risk minimization, it remains true in this context that
\begin{equation}
\label{eqn:ag-vs-uc-ordered}
\AgRate^{o}(n,k) \leq 2\UC^{o}(n,k)
\end{equation}
and an immediate implication is then that any $n,k$ has
\begin{equation*}
\AgRate^{o}(n,k) \lesssim \sqrt{\frac{k \log(n)}{n}}.
\end{equation*}

As above, it is interesting to ask whether the $\log(n)$ factor is required is necessary.
Analogously to the order-invariant compression schemes above, we find that it is indeed necessary, 
as stated in the following theorem.  Note that this lower bound for order-dependent compression schemes 
is slightly larger than that established above for order-independent compression schemes.

\begin{theorem}
\label{thm:ag-lb-order-dep}
For any $n,k \in \nats$ with $|\X| \geq n \geq c k \log(n)$,
\begin{equation*}
\AgRate^{o}(n,k) \gtrsim \sqrt{\frac{k \log(n)}{n}}.
\end{equation*}
\end{theorem}

Together with \eqref{eqn:ag-vs-uc-ordered}, this has the following immediate implication.

\begin{theorem}
\label{thm:uc-lb-order-dep}
For any $n,k \in \nats$ with $|\X| \geq n \geq c k \log(n)$,
\begin{equation*}
\UC^{o}(n,k) \gtrsim \sqrt{\frac{k \log(n)}{n}}.
\end{equation*}
\end{theorem}

As above, combining this with the known upper bound, this provides a tight characterization of the worst-case 
rate of uniform convergence for order-dependent agnostic sample compression schemes.

\begin{corollary}
\label{cor:uc-tight-order-dep}
For any $n,k \in \nats$ with $|\X| \geq n \geq c k \log(n)$, 
\begin{equation*}
\AgRate^{o}(n,k) \asymp \sqrt{\frac{k \log(n)}{n}}
\end{equation*}
and
\begin{equation*}
\UC^{o}(n,k) \asymp \sqrt{\frac{k \log(n)}{n}}.
\end{equation*}
\end{corollary}

We now present the proof of Theorem~\ref{thm:ag-lb-order-dep}.

\begin{proof}[Proof of Theorem~\ref{thm:ag-lb-order-dep}]
The construction used in this proof is analogous to that from the proof of Theorem~\ref{thm:ag-lb}, 
and in fact is slightly simpler.
Fix any $n,k \in \nats$ with $|\X| \geq n \geq c k \log_{2}(n)$ for a sufficiently large numerical constant $c \geq 4$ (discussed below).
The essential strategy here is the same as in the permutation-invariant compression schemes, 
in that we are constructing an embedded agnostic learning problem for a constructed VC class, 
but in this case the VC dimension will be larger: $k \log_{2}(m)$, with $m \approx n$.
Specifically, let $m = 2^{\lfloor \log_{2}( n ) \rfloor}$, 
and let $x_{0},\ldots,x_{m-1}$ denote any $m$ distinct elements of $\X$.
For simplicity, suppose $\frac{m}{k\log_{2}(m)} \in \nats$ (as before, the argument easily extends to the general case by introducing floor functions, and only the numerical constants change).

We break the space up into \emph{blocks} as before, but now for each $t \in \{1,\ldots,k\}$
we let $B_{t} = \left\{ (t-1) \frac{m}{k}, \ldots, t \frac{m}{k} - 1 \right\}$, 
and for each $s \in \{1, \ldots, m / (k\log_{2}(m))\}$ we define a \emph{sub-block} 
\begin{equation*}
B_{ts} = \left\{ (t-1) \frac{m}{k} + (s-1)\log_{2}(m), \ldots, (t-1) \frac{m}{k} + s \log_{2}(m) - 1 \right\}.
\end{equation*}
Thus, as before, a sub-block consists of $\log_{2}(m)$ indices, 
but now a block only contains ${m}/{k}$ indices, and hence $\frac{m}{k\log_{2}(m)}$ sub-blocks.
Now for $t \in \{1,\ldots,k\}$ and $i \in \{0,\ldots, m-1\}$, define a classifier $h_{t,i} : \X \to \Y$ 
with the property that, $\forall s \in \{1,\ldots, m / (k\log_{2}(m))\}$, 
$\forall r \in \{0,\ldots,\log_{2}(m)-1\}$, for $j = (t-1) \frac{m}{k} + (s-1) \log_{2}(m) + r$, 
\begin{equation*}
h_{t,i}(x_{j}) = b_{r}(i),
\end{equation*}
where as above, $b_{r}(i)$ is the $(r+1)^{{\rm th}}$ bit in the binary representation of $i$: 
i.e., $i = \sum_{\ell = 0}^{\log_{2}(m)-1} b_{\ell}(i) 2^{\ell}$, with $b_{0}(i),\ldots,b_{\log_{2}(m)-1}(i) \in \{0,1\}$.
Thus, the index $i$ encodes the prediction values for the points $\{ x_{\ell} : \ell \in B_{ts} \}$ 
as the bits of $i$; this is slightly different from the $h_{t,i}$ functions we defined above, 
since $i$ is already in $\{0,\ldots,m-1\}$ here, so there is no need to subtract anything from it.

Now we construct a reconstruction function $\recon$ that outputs functions which again 
correspond to some such $h_{t,i}$ function within each block $B_{t}$, 
and which potentially uses a different bit pattern $i$ for each $t$.
Formally, for any 
$i_{1},\ldots,i_{k} \in \{0,\ldots,m-1\}$ 
and any $y_{1},\ldots,y_{k} \in \Y$,
define $\recon(((x_{i_{1}},y_{1}),\ldots,(x_{i_{k}},y_{k}))) = \tilde{h}_{i_{1},\ldots,i_{k}}$, 
where here $\tilde{h}_{i_{1},\ldots,i_{k}} : \X \to \Y$ is any function satisfying the property that
each $t \in \{1,\ldots,k\}$ and $j \in \{(t-1)m,\ldots,tm-1\}$ has 
$\tilde{h}_{i_{1},\ldots,i_{k}}(x_{j}) = h_{t,i_{t}}(x_{j})$: that is, the points $x_{i_{t}}$ in the compression set 
are interpreted by the compression scheme as encoding the desired label sequence for sub-blocks $B_{ts}$ in the \emph{bits} of $i_{t}$.
Note that unlike the order-independent compression scheme construction, we do not require $i_{t}$ to be in block $B_{t}$.
Instead, we are able to distinguish which $i_{t}$ to use to specify the $h_{t,i_{t}}$ sub-predictor for block $B_{t}$ 
simply using the \emph{order} of the sequence $((x_{i_{1}},y_{1}),\ldots,(x_{i_{k}},y_{k}))$.
For our purposes, $\tilde{h}_{i_{1},\ldots,i_{k}}(x)$ may be defined arbitrarily for $x \in \X \setminus \{x_{0},\ldots,x_{m-1}\}$.
Again, since $\recon(((x_{i_{1}},y_{1}),\ldots,(x_{i_{k}},y_{k})))$ is invariant to the $y_{1},\ldots,y_{k}$ values, 
for brevity we will drop the $y_{i}$ arguments and simply write $\recon((x_{i_{1}},\ldots,x_{i_{k}}))$.
For completeness, $\recon(S)$ should also be defined for sequences $S$ of length strictly less than $k$, 
or sequences containing elements not in $\{x_{0},\ldots,x_{m-1}\}$;
for our purposes, in these cases, if $k^{\prime}$ of the elements in $S$ \emph{are} contained in $\{x_{0},\ldots,x_{m-1}\}$, 
then enumerate them as $x_{i_{1}^{\prime}},\ldots,x_{i_{k^{\prime}}^{\prime}}$; then if $k^{\prime} < k$, let $i_{k^{\prime}+1}^{\prime} = \cdots = i_{k}^{\prime} = 0$, 
and finally define the output of $\recon(S)$ as $\tilde{h}_{i_{1}^{\prime},\ldots,i_{k}^{\prime}}$: that is, it interprets the sub-sequence of points in $S$ contained in $\{x_{0},\ldots,x_{\log_{2}(m)-1}\}$ 
as the initial indices $i_{t}$, and fills in the rest of the indices up to $i_{k}$ using $0$'s.

Now define a family of distributions $P^{(\sigma)}$, $\sigma = \{\sigma_{t,r}\}$, 
with $\sigma_{t,r} \in \{-1,1\}$ for $t \in \{1,\ldots,k\}$ and $r \in \{0,\ldots,\log_{2}(m)-1\}$, as follows.
Every $P^{(\sigma)}$ has marginal $P_{X}$ on $\X$ uniform on $x_{0},\ldots,x_{m-1}$, 
and for each $j = (t-1)\frac{m}{k} + (s-1)\log_{2}(m) + r \in B_{ts}$ 
(for $t \in \{1,\ldots,k\}$, $s \in \{1,\ldots, m/(k\log_{2}(m))\}$, and $r \in \{0,\ldots,\log_{2}(m)-1\}$)
set $P^{(\sigma)}(Y=1|X=x_{j}) = \frac{1}{2} + \frac{\epsilon}{2} \sigma_{t,r}$,
where 
\begin{equation*}
\epsilon = \sqrt{\frac{k \log_{2}(m)}{n}}.
\end{equation*}
Now let us suppose $\sigma$ is chosen \emph{randomly}, with $\sigma_{t,r}$ independent ${\rm Uniform}(\{-1,1\})$.
Then 
\begin{equation*}
\AgRate^{o}(n,k) \geq \E\!\left[ \inf_{\comp} \AgRate^{o}(n,k,\recon,\comp,P^{(\sigma)}) \right],
\end{equation*}
so that it suffices to lower-bound the expression on the right hand side.

For any $t \in \{1,\ldots,k\}$ and $r \in \{0,\ldots,\log_{2}(m)-1\}$, denote 
\begin{equation*}
C_{t,r} = \{ (t-1)\frac{m}{k} + (s-1)\log_{2}(m) + r : s \in \{1,\ldots,m/(k\log_{2}(m))\}\}.
\end{equation*}
Also define 
$\H_{k,\recon}^{*} = \{ \recon((x_{i_{1}},\ldots,x_{i_{k}})) : i_{1},\ldots,i_{k} \in \{0,\ldots,m-1\}\}$, 
the space of all possible classifiers $\recon$ can produce.
As before, we are concerned both with constructing a lower bound on the excess risk of $\hat{h} = \recon(\comp(Z_{[n]}))$ 
relative to $\min_{h \in \H_{k,\recon}^{*}} R(h;P^{(\sigma)})$ via a traditional VC lower bound argument, 
and also with upper-bounding $\min_{h \in \H_{k,\recon}(Z_{[n]})} R(h;P^{(\sigma)}) - \min_{h \in \H_{k,\recon}^{*}} R(h;P^{(\sigma)})$, 
so that the lower bound remains nearly valid for the excess risk of $\hat{h}$ relative to classifiers $\recon$ 
can actually produce given sequences within this data set $Z_{[n]}$.

Fix any $t \in \{1,\ldots,k\}$ and let $P_{t}^{(\sigma)}$ denote the conditional distribution of $(X,Y) \sim P^{(\sigma)}$ 
given $\sigma$ and the event that $X \in \{ x_{j} : j \in B_{t} \}$.  Also denote 
$\H_{t}^{*} = \{ h_{t,i} : i \in \{0,\ldots,m-1\} \}$, 
$i_{t}^{*} = \argmin_{i \in \{0,\ldots,m-1\}} R(h_{t,i};P_{t}^{(\sigma)})$
$h_{t}^{*} = h_{t,i_{t}^{*}}$,
and 
\begin{equation*}
\H_{t}(Z_{[n]}) = \{ h_{t,i} : i \in \{0,\ldots,m-1\}, x_{i} \in \{X_{1},\ldots,X_{n}\} \}.
\end{equation*}
As before, we are now interested in proving that any compression function $\comp$ 
results in $\hat{h} = \recon(\comp(Z_{[n]}))$ with 
$\E[ R(\hat{h};P_{t}^{(\sigma)}) - R(h_{t}^{*};P_{t}^{(\sigma)}) ] \geq \epsilon/(8 e^{4})$, 
and also that 
$\E[ \min_{h \in \H_{t}(Z_{[n]})} R(h;P_{t}^{(\sigma)}) - R(h_{t}^{*};P_{t}^{(\sigma)}) ] \leq \epsilon / (16 e^{4})$.

The first part proceeds nearly identically to the corresponding part in the proof of Theorem~\ref{thm:ag-lb}, 
with a few changes needed to convert to this scenario. 
For any $r \in \{0,\ldots,\log_{2}(m)-1\}$, 
for any $j \in C_{t,r}$, note that $h_{t}^{*}(x_{j}) = \frac{\sigma_{t,r}+1}{2}$.
Also, for any compression function $\comp$, 
any $\hat{h}$ that $\recon(\comp(Z_{[n]}))$ is capable of producing 
has $\hat{h}(x_{j}) = \hat{h}(x_{j^{\prime}})$ 
for every $j,j^{\prime} \in C_{t,r}$.
In particular, if we let $\hat{i}_{t} \in \{0,\ldots,m-1\}$ be the index with 
$b_{r}(\hat{i}_{t}) = \hat{h}(x_{(t-1)(m/2)+r})$ for every $r \in \{0,\ldots,\log_{2}(m)-1\}$, 
then $\hat{h}$ and $h_{t,\hat{i}_{t}}$ agree on every element of $\{x_{j} : j \in B_{t}\}$.
This also implies 
\begin{align*}
R(\hat{h};P_{t}^{(\sigma)}) - R(h_{t}^{*};P_{t}^{(\sigma)})
& = R(h_{t,\hat{i}_{t}};P_{t}^{(\sigma)}) - R(h_{t}^{*};P_{t}^{(\sigma)})
\\ & = \frac{1}{\log_{2}(m)} \sum_{r=0}^{\log_{2}(m)-1} \epsilon \ind\!\left[ b_{r}(\hat{i}_{t}) \neq \frac{\sigma_{t,r}+1}{2} \right].
\end{align*}
Therefore, denoting by $n_{t,r} = |\{ i \leq n : X_{i} \in \{ x_{j} : j \in C_{t,r} \} \}|$, we have 
\begin{equation*}
\E[ R(\hat{h};P_{t}^{(\sigma)}) - R(h_{t}^{*};P_{t}^{(\sigma)}) ] 
= \frac{\epsilon}{\log_{2}(m)} \sum_{r=0}^{\log_{2}(m)-1} \E\!\left[ \P\!\left( b_{r}(\hat{i}_{t}) \neq \frac{\sigma_{t,r}+1}{2} \middle| n_{t,r} \right) \right].
\end{equation*}
For any $r \in \{0,\ldots,\log_{2}(m)-1\}$, 
enumerate the $n_{t,r}$ random variables $(X_{i},Y_{i})$ with $X_{i} \in \{ x_{j} : j \in C_{t,r} \}$ as $(X_{i(r,1)},Y_{i(r,1)}),\ldots,(X_{i(r,n_{t,r})},Y_{i(r,n_{t,r})})$,
and note that given $n_{t,r}$, the values $(Y_{i(r,1)},\ldots,Y_{i(r,n_{t,r})})$ are a \emph{sufficient statistic} for $\sigma_{t,r}$ (see Definition 2.4 of \cite{schervish:95}), 
and therefore (see Theorem 3.18 of \cite{schervish:95}) there exists a (randomized) decision rule $\hat{f}_{t,r}(Y_{i(r,1)},\ldots,Y_{i(r,n_{t,r})})$ depending only on these variables 
and independent random bits such that
\begin{equation*}
\P\!\left( b_{r}(\hat{i}_{t}) \neq \frac{\sigma_{t,r}+1}{2} \middle| n_{t,r} \right)
= \P\!\left( \hat{f}_{t,r}(Y_{i(r,1)},\ldots,Y_{i(r,n_{t,r})}) \neq \frac{\sigma_{t,r}+1}{2} \middle| n_{t,r} \right).
\end{equation*}
Furthermore, by Lemma 5.1 of \cite{MR1741038}, we have 
\begin{equation*}
\P\!\left( \hat{f}_{t,r}(Y_{i(r,1)},\ldots,Y_{i(r,n_{t,r})}) \neq \frac{\sigma_{t,r}+1}{2} \middle| n_{t,r} \right) 
> \frac{1}{8e} \exp\!\left\{ - (8/3) n_{t,r} \epsilon^{2} \right\}.
\end{equation*}
Altogether, and combined with Jensen's inequality, we have that 
\begin{align*}
& \E[ R(\hat{h};P_{t}^{(\sigma)}) - R(h_{t}^{*};P_{t}^{(\sigma)}) ] 
\\ & \geq \frac{\epsilon}{8e\log_{2}(m)} \sum_{r=0}^{\log_{2}(m)-1} \E\!\left[ \exp\!\left\{ - (8/3) n_{t,r} \epsilon^{2} \right\} \right]
\geq \frac{\epsilon}{8e\log_{2}(m)} \sum_{r=0}^{\log_{2}(m)-1} \exp\!\left\{ - (8/3) \E[ n_{t,r} ] \epsilon^{2} \right\}
\\ & = \frac{\epsilon}{8e\log_{2}(m)} \sum_{r=0}^{\log_{2}(m)-1} \exp\!\left\{ - (8/3) \frac{n}{k\log_{2}(m)} \epsilon^{2} \right\}
\geq \frac{\epsilon}{8e\log_{2}(m)} \sum_{r=0}^{\log_{2}(m)-1} e^{ - (8/3) }
\geq \frac{\epsilon}{8 e^{4}}.
\end{align*}

Next, continuing on to the second part, 
for any $x \in \{x_{i} : i \in \{0,\ldots,m-1\}\}$, denote by $I(x)$ the index $i$ such that $x = x_{i}$.
Similarly to before, an $i$ for which $h_{t,i}$ has minimal $R(h_{t,i};P_{t}^{(\sigma)})$ among all $h_{t,i^{\prime}} \in \H_{t}(Z_{[n]})$
can equivalently be defined as an $i$ 
with minimal $\sum_{j=0}^{\log_{2}(m)-1} \ind[ b_{j}(i) \neq b_{j}(i_{t}^{*}) ]$ 
among all $i^{\prime} \in \{ I(X_{1}),\ldots,I(X_{n}) \}$, 
and furthermore, for such an $i$, 
\begin{equation*}
R(h_{t,i};P_{t}^{(\sigma)}) - R(h_{t}^{*};P_{t}^{(\sigma)}) 
= \frac{\epsilon}{\log_{2}(m)} \sum_{j=0}^{\log_{2}(m)-1} \ind[ b_{j}(i) \neq b_{j}(i_{t}^{*}) ].
\end{equation*}
For any $i \in \{0,\ldots,m-1\}$, denote 
$\Delta_{t}(i) = \sum_{j=0}^{\log_{2}(m)-1} \ind[ b_{j}(i) \neq b_{j}(i_{t}^{*}) ]$.
It therefore suffices to prove an upper bound for the quantity 
\begin{equation*}
\frac{\epsilon}{\log_{2}(m)} \E\!\left[ \min_{i \in \{ I(X_{1}),\ldots,I(X_{n}) \}} \Delta_{t}(i) \right].
\end{equation*}
Define a random variable $X$ with distribution $P_{X}$ (recalling that this is uniform on $\{x_{0},\ldots,x_{m-1}\}$).
Then the conditional distribution of $\Delta_{t}(I(X))$ given $\sigma$ 
is ${\rm Binomial}(\log_{2}(m),\frac{1}{2})$.
Letting $q = 16 e^{4}$, and supposing $c$ is sufficiently large so that $q \leq (1/2)\log_{2}(m)$, 
following the argument from the analogous step in the proof of Theorem~\ref{thm:ag-lb} 
(where an analysis is given that would apply to \emph{any} ${\rm Binomial}(\log_{2}(m),\frac{1}{2})$ random variable) we have 
\begin{equation*}
\P\!\left( \Delta_{t}(I(X)) \leq \frac{1}{2q}\log_{2}(m) \middle| \sigma \right) \geq m^{(1/2q)\log_{2}(4q) - 1},
\end{equation*}
which implies (still following similar derivations as in the proof of Theorem~\ref{thm:ag-lb}, except with $n_{t}$ replaced by $n$) 
\begin{equation*}
\P\!\left( \min_{i \in \{I(X_{1}),\ldots,I(X_{n})\}} \Delta_{t}(i) > \frac{1}{2q}\log_{2}(m) \middle| \sigma \right) \leq \exp\!\left\{  - m^{(1/2q)\log_{2}(4q)-1} n \right\}.
\end{equation*}
By the law of total expectation and the fact that $R(h;P_{t}^{(\sigma)}) \leq 1$, 
we have 
\begin{equation*}
\E\!\left[ \min_{h \in \H_{t}(Z_{[n]})} R(h;P_{t}^{(\sigma)}) - R(h_{t}^{*};P_{t}^{(\sigma)}) \right] 
\leq \frac{\epsilon}{2q} + \exp\!\left\{  - m^{(1/2q)\log_{2}(4q)-1} n \right\}.
\end{equation*}
Then note that 
\begin{align*}
& \exp\!\left\{  - m^{(1/2q)\log_{2}(4q)-1} n \right\}
\leq \exp\!\left\{  - m^{(1/2q)\log_{2}(4q)} \right\}
\\ & = \left( \exp\!\left\{  - (1/2q)\log_{2}(4q) m^{(1/2q)\log_{2}(4q)} \right\}  \right)^{\frac{2q}{\log_{2}(4q)}}
\\ & \leq \left( \frac{2q}{\log_{2}(4q)} \frac{1}{m^{(1/2q)\log_{2}(4q)}} \right)^{\frac{2q}{\log_{2}(4q)}}
= \left( \frac{2q}{\log_{2}(4q)} \right)^{\frac{2q}{\log_{2}(4q)}} \frac{1}{m}.
\end{align*}
Since this last expression shrinks strictly faster than the above specification of $\epsilon$ as a function of $n/(k\log(n))$, 
we may conclude that for a sufficiently large choice of the numerical constant $c$, 
this expression is smaller than $\frac{\epsilon}{32 e^{4}}$.
Therefore, we conclude that 
\begin{equation*}
\E\!\left[ \min_{h \in \H_{t}(Z_{[n]})} R(h;P_{t}^{(\sigma)}) - R(h_{t}^{*};P_{t}^{(\sigma)}) \right] \leq \frac{\epsilon}{16 e^{4}}.
\end{equation*}

These two parts combine to imply that
\begin{align*}
& \E\!\left[ R(\hat{h};P_{t}^{(\sigma)}) - \min_{h \in \H_{t}(Z_{[n]})} R(h;P_{t}^{(\sigma)}) \right]
\\ & = \E\!\left[ R(\hat{h};P_{t}^{(\sigma)}) - R(h_{t}^{*};P_{t}^{(\sigma)})\right] - \E\!\left[ \min_{h \in \H_{t}(Z_{[n]})} R(h;P_{t}^{(\sigma)}) - R(h_{t}^{*};P_{t}^{(\sigma)}) \right]
\geq \frac{\epsilon}{16 e^{4}}.
\end{align*}

As a final step, we stitch together these lower bounds for the blocks to create a lower bound under the full distribution $P^{(\sigma)}$.
Toward this end, note that any $h$ has 
$R(h;P^{(\sigma)}) = \frac{1}{k} \sum_{t=1}^{k} R(h;P_{t}^{(\sigma)})$.
Also note that, for this reconstruction function $\recon$, 
every $\tilde{h}_{i_{1},\ldots,i_{k}}$ function with $i_{1},\ldots,i_{k} \in \{I(X_{1}),\ldots,I(X_{n})\}$ 
can be produced by $\recon$ using an argument sequence $S$ of at most $k$ elements of $\{X_{1},\ldots,X_{n}\}$: 
namely, $S = (x_{i_{1}},\ldots,x_{i_{k}})$, since each of these $x_{i_{t}}$ are in $\{X_{1},\ldots,X_{n}\}$ due to $i_{t} \in \{I(X_{1}),\ldots,I(X_{n})\}$.
Also note that $R(\tilde{h}_{i_{1},\ldots,i_{k}};P_{t}^{(\sigma)}) = R(h_{t,i_{t}};P_{t}^{(\sigma)})$.
Therefore, 
\begin{align*}
& \min_{h \in \H_{k,\recon}(Z_{[n]})} R(h;P^{(\sigma)}) 
\leq \min_{i_{1},\ldots,i_{k} \in \{I(X_{1}),\ldots,I(X_{n})\}} R(\tilde{h}_{i_{1},\ldots,i_{k}};P^{(\sigma)})
\\ & = \min_{i_{1},\ldots,i_{k} \in \{I(X_{1}),\ldots,I(X_{n})\}} \frac{1}{k} \sum_{t=1}^{k} R(h_{t,i_{t}};P_{t}^{(\sigma)})
\\ & = \frac{1}{k} \sum_{t=1}^{k} \min_{i_{t}  \in \{I(X_{1}),\ldots,I(X_{n})\}} R(h_{t,i_{t}};P^{(\sigma)}_{t})
= \frac{1}{k} \sum_{t=1}^{k} \min_{h \in \H_{t}(Z_{[n]})} R(h;P^{(\sigma)}_{t}).
\end{align*}
Thus, for any compression function $\comp$, 
denoting $\hat{h} = \recon(\comp(Z_{[n]}))$, 
\begin{align*}
& \E\!\left[ R(\hat{h};P^{(\sigma)}) - \min_{h \in \H_{k,\recon}(Z_{[n]})} R(h;P^{(\sigma)}) \right] 
\\ & \geq \frac{1}{k} \sum_{t=1}^{k} \E\!\left[ R(\hat{h};P^{(\sigma)}_{t}) - \min_{h \in \H_{t}(Z_{[n]})} R(h;P_{t}^{(\sigma)}) \right]
\geq \frac{1}{16 e^{4}} \epsilon 
\gtrsim \sqrt{\frac{k \log(n)}{n}}.
\end{align*}
\end{proof}

\bibliographystyle{plainnat}
\bibliography{ourbib}

\begin{thebibliography}{21}
\providecommand{\natexlab}[1]{#1}
\providecommand{\url}[1]{\texttt{#1}}
\expandafter\ifx\csname urlstyle\endcsname\relax
  \providecommand{\doi}[1]{doi: #1}\else
  \providecommand{\doi}{doi: \begingroup \urlstyle{rm}\Url}\fi

\bibitem[Anthony and Bartlett(1999)]{MR1741038}
Martin Anthony and Peter~L. Bartlett.
\newblock \emph{{N}eural {N}etwork {L}earning: {T}heoretical {F}oundations}.
\newblock Cambridge University Press, Cambridge, 1999.
\newblock ISBN 0-521-57353-X.
\newblock \doi{10.1017/CBO9780511624216}.
\newblock URL \url{http://dx.doi.org/10.1017/CBO9780511624216}.

\bibitem[Blumer et~al.(1989)Blumer, Ehrenfeucht, Haussler, and
  Warmuth]{MR1072253}
Anselm Blumer, Andrzej Ehrenfeucht, David Haussler, and Manfred~K. Warmuth.
\newblock Learnability and the {V}apnik-{C}hervonenkis dimension.
\newblock \emph{J. Assoc. Comput. Mach.}, 36\penalty0 (4):\penalty0 929--965,
  1989.
\newblock ISSN 0004-5411.

\bibitem[David et~al.(2016)David, Moran, and Yehudayoff]{david2016supervised}
Ofir David, Shay Moran, and Amir Yehudayoff.
\newblock Supervised learning through the lens of compression.
\newblock In \emph{Advances in Neural Information Processing Systems}, pages
  2784--2792, 2016.

\bibitem[Devroye et~al.(1996)Devroye, Gy{\"o}rfi, and Lugosi]{MR1383093}
Luc Devroye, L{\'a}szl{\'o} Gy{\"o}rfi, and G{\'a}bor Lugosi.
\newblock \emph{A probabilistic theory of pattern recognition}, volume~31 of
  \emph{Applications of Mathematics (New York)}.
\newblock Springer-Verlag, New York, 1996.
\newblock ISBN 0-387-94618-7.

\bibitem[Floyd and Warmuth(1995)]{DBLP:journals/ml/FloydW95}
Sally Floyd and Manfred~K. Warmuth.
\newblock Sample compression, learnability, and the vapnik-chervonenkis
  dimension.
\newblock \emph{Machine Learning}, 21\penalty0 (3):\penalty0 269--304, 1995.

\bibitem[Gottlieb et~al.(2014)Gottlieb, Kontorovich, and
  Nisnevitch]{DBLP:conf/nips/GottliebKN14}
Lee{-}Ad Gottlieb, Aryeh Kontorovich, and Pinhas Nisnevitch.
\newblock Near-optimal sample compression for nearest neighbors.
\newblock In \emph{Advances in Neural Information Processing Systems 27: Annual
  Conference on Neural Information Processing Systems 2014, December 8-13 2014,
  Montreal, Quebec, Canada}, pages 370--378, 2014.

\bibitem[Graepel et~al.(2005)Graepel, Herbrich, and
  Shawe-Taylor]{DBLP:journals/ml/GraepelHS05}
Thore Graepel, Ralf Herbrich, and John Shawe-Taylor.
\newblock Pac-bayesian compression bounds on the prediction error of learning
  algorithms for classification.
\newblock \emph{Machine Learning}, 59\penalty0 (1-2):\penalty0 55--76, 2005.

\bibitem[Hanneke(2016)]{DBLP:journals/jmlr/Hanneke16}
Steve Hanneke.
\newblock The optimal sample complexity of {PAC} learning.
\newblock \emph{Journal of Machine Learning Research}, 17:\penalty0
  38:1--38:15, 2016.
\newblock URL \url{http://jmlr.org/papers/v17/15-389.html}.

\bibitem[Hanneke et~al.(2018)Hanneke, Kontorovich, and Sadigurschi]{hks18}
Steve Hanneke, Aryeh Kontorovich, and Menachem Sadigurschi.
\newblock Efficient conversion of learners to bounded sample compressors.
\newblock 2018.

\bibitem[Haussler(1995)]{MR1313896}
David Haussler.
\newblock Sphere packing numbers for subsets of the {B}oolean {$n$}-cube with
  bounded {V}apnik-{C}hervonenkis dimension.
\newblock \emph{J. Combin. Theory Ser. A}, 69\penalty0 (2):\penalty0 217--232,
  1995.

\bibitem[Haussler et~al.(1994)Haussler, Littlestone, and
  Warmuth]{DBLP:journals/iandc/HausslerLW94}
David Haussler, Nick Littlestone, and Manfred~K. Warmuth.
\newblock Predicting \{0,1\}-functions on randomly drawn points.
\newblock \emph{Inf. Comput.}, 115\penalty0 (2):\penalty0 248--292, 1994.
\newblock \doi{10.1006/inco.1994.1097}.

\bibitem[Kane et~al.(2017)Kane, Livni, Moran, and
  Yehudayoff]{DBLP:journals/corr/abs-1711-05893}
Daniel~M. Kane, Roi Livni, Shay Moran, and Amir Yehudayoff.
\newblock On communication complexity of classification problems.
\newblock \emph{CoRR}, abs/1711.05893, 2017.
\newblock URL \url{http://arxiv.org/abs/1711.05893}.

\bibitem[Koltchinskii(1981)]{koltchinskii1981central}
Vladimir~I. Koltchinskii.
\newblock On the central limit theorem for empirical measures.
\newblock \emph{Theory of Probability and Mathematical Statistics},
  24:\penalty0 71--82, 1981.

\bibitem[Kontorovich and Pinelis(2016)]{DBLP:journals/corr/KontorovichP16}
Aryeh Kontorovich and Iosif Pinelis.
\newblock Exact lower bounds for the agnostic probably-approximately-correct
  {(PAC)} machine learning model.
\newblock \emph{CoRR}, abs/1606.08920, 2016.
\newblock URL \url{http://arxiv.org/abs/1606.08920}.

\bibitem[Littlestone and Warmuth(1986)]{warmuth86}
Nick Littlestone and Manfred~K. Warmuth.
\newblock Relating data compression and learnability, unpublished.
\newblock 1986.

\bibitem[Moran and Yehudayoff(2016)]{DBLP:journals/jacm/MoranY16}
Shay Moran and Amir Yehudayoff.
\newblock Sample compression schemes for {VC} classes.
\newblock \emph{J. {ACM}}, 63\penalty0 (3):\penalty0 21:1--21:10, 2016.
\newblock \doi{10.1145/2890490}.
\newblock URL \url{http://doi.acm.org/10.1145/2890490}.

\bibitem[Pollard(1982)]{pollard1982central}
David Pollard.
\newblock A central limit theorem for empirical processes.
\newblock \emph{Journal of the {A}ustralian Mathematical Society}, 33\penalty0
  (2):\penalty0 235--248, 1982.

\bibitem[Schervish(1995)]{schervish:95}
Mark~J. Schervish.
\newblock \emph{Theory of Statistics}.
\newblock Springer-Verlag, 1995.

\bibitem[Talagrand(1994)]{talagrand1994}
Michel Talagrand.
\newblock Sharper bounds for gaussian and empirical processes.
\newblock \emph{Ann. Probab.}, 22\penalty0 (1):\penalty0 28--76, 01 1994.
\newblock \doi{10.1214/aop/1176988847}.
\newblock URL \url{http://dx.doi.org/10.1214/aop/1176988847}.

\bibitem[Vapnik and Chervonenkis(1971)]{vapnik:71}
V.~Vapnik and A.~Chervonenkis.
\newblock On the uniform convergence of relative frequencies of events to their
  probabilities.
\newblock \emph{Theory of Probability and its Applications}, 16\penalty0
  (2):\penalty0 264--280, 1971.

\bibitem[Zhivotovskiy(2017)]{DBLP:conf/colt/Zhivotovskiy17}
Nikita Zhivotovskiy.
\newblock Optimal learning via local entropies and sample compression.
\newblock In \emph{{COLT} Conference on Learning Theory}, 2017.

\end{thebibliography}

\end{document}